\documentclass[twoside]{article}

\usepackage[preprint]{aistats2026}
\runningtitle{Learning Invariant Graph Representations Through Redundant Information}
\runningauthor{Learning Invariant Graph Representations Through Redundant Information}
\usepackage[round]{natbib}

\usepackage{xcolor}
\usepackage{graphicx}
\usepackage{amsmath}
\usepackage{amsthm}
\usepackage{placeins}
\usepackage{thmtools}
\usepackage{thm-restate}
\usepackage{amssymb}
\usepackage{enumitem}
\usepackage{algorithm2e}
\usepackage{subcaption}

\usepackage{amsmath,amsthm,thmtools}
\usepackage{amsfonts}
\usepackage{amssymb}
\newcommand{\argmax}[1]{\underset{#1}{\arg\max}}

\newcommand{\indep}{\perp \!\!\! \perp}
\newcommand{\uni}[2]{\mathrm{Uni}({#1{:}#2})}
\newcommand{\rdn}[2]{\mathrm{Red}({#1{:}#2})}
\newcommand{\rdncap}[2]{\mathrm{Red_\cap}({#1{:}#2})}
\newcommand{\syn}[2]{\mathrm{Syn}({#1{:}#2})}
\newcommand{\mut}[2]{\mathrm{I}({#1;#2})}
\newcommand{\mutd}[3]{\mathrm{I}_{#1}({#2;#3})}

\newtheorem{lemma}{Lemma}

\newtheorem{definition}{Definition}
\newtheorem{prop}{Proposition}

\newtheorem{propopt}{Proposed Optimization}

\usepackage{booktabs}

\begin{document}

%

%

\twocolumn[

\aistatstitle{Learning Invariant Graph Representations Through\\ Redundant Information}
\vspace{-15pt}
\aistatsauthor{Barproda Halder \And Pasan Dissanayake  \And  Sanghamitra Dutta }

\aistatsaddress{University of Maryland, College Park} ]

\begin{abstract}

Learning invariant graph representations for out-of-distribution (OOD) generalization remains challenging because the learned representations often retain spurious components. To address this challenge, this work introduces a new tool from information theory called Partial Information Decomposition (PID) that goes beyond classical information-theoretic measures. We identify limitations in existing approaches for invariant representation learning that solely rely on classical information-theoretic measures, motivating the need to precisely focus on redundant information about the target $Y$ shared between spurious subgraphs $G_s$ and invariant subgraphs $G_c$ obtained via PID. Next, we propose a new multi-level optimization framework that we call -- Redundancy-guided Invariant Graph learning (RIG) -- that maximizes redundant information while isolating spurious and causal subgraphs, enabling OOD generalization under diverse distribution shifts. Our approach relies on alternating between estimating a lower bound of redundant information (which itself requires an optimization) and maximizing it along with additional objectives. Experiments on both synthetic and real-world graph datasets demonstrate the generalization capabilities of our proposed  RIG framework.

\let\thefootnote\relax\footnotetext{
Correspondence to: B. Halder $<$bhalder@umd.edu$>$.  Presented at WiML Workshop @ NeurIPS 2025.
}
\end{abstract}

\section{Introduction}
Graph Neural Networks (GNNs) have achieved substantial strides in learning from structured data, driving significant advances in a wide range of applications~\citep{kipf2016semi,wu2020comprehensive,dai2024comprehensive}. Despite their success, a critical limitation remains: \emph{GNNs trained on one data distribution fail to generalize well to real-world distribution shifts}. Such shifts can occur due to changes in data collection environments or data generation processes~\citep{ji2023drugood,zou2023gdl,gui2022good}. Distribution shifts can also spuriously correlate with target labels, 
 leading to substantial performance degradation when models are deployed in out-of-distribution (OOD) real-world settings \citep{li2022ood,fan2023generalizing,guo2024investigating}. Thus, OOD generalization is essential for the reliable deployment of GNNs.

Although OOD generalization has been extensively studied in Euclidean domains such as images \citep{ahuja2021invariance,IRM}, applying it on graphs is particularly challenging for two main reasons. First, the distribution shifts on the graphs are complicated. They can occur at both the attribute level and the structure level and can also spuriously correlate with the target labels \citep{yehudai2021local,IRM,GIB,nagarajan2020understanding}. Second, the unavailability of environment or domain labels makes the generalization even harder \citep{hu2020open,ciga,gala}. Recent progress in OOD learning show that predictive models that can solely focus on causal factors of the target can remain robust under a wide range of distributional changes. However, the challenges associated with graph data prohibit the direct adoption of such causal methods~\citep{dir,ciga,fan2022debiasing,li2022learning,yang2022learning}. Existing works~\citep{miao2022interpretable1,GIB,ciga,gala} often incorporate information-theoretic measures to have robust objective functions for improving generalization under distribution shifts. However, achieving feature invariance across varying distribution shifts remains a difficult problem. 

To address the challenge of OOD generalization, we study the integration of invariant graph representation learning with Partial Information Decomposition (PID)~\citep{williams2010nonnegativePID,bertschinger2014quantifying}, an emerging body of work from information theory that goes beyond classical measures like mutual information, conditional mutual information, etc. PID specifically explains the structure of multivariate information, disentangling the joint mutual information $\mut{Y}{C,S}$ in invariant variable $C$ and spurious variable $S$ about target $Y$ into four non-negative terms: uniqueness (in $C$ or $S$), redundancy (common knowledge between $C$ and $S$), and synergy (manifests only when $C$ and $S$ are together).
We seek to address the following question: \emph{Can decomposing the multivariate information between spurious and invariant subgraphs assist in achieving improved generalization in GNNs?}  

We begin by employing Structural Causal Models (SCMs) \citep{pearl2009causality} to characterize the graph generation process under distribution shifts and analyze the interactions between spurious and invariant subgraphs. Building on this causal perspective, we establish theoretical connections between SCMs and PID components by analyzing canonical examples. Our analysis identifies limitations of existing techniques that solely rely on classical information-theoretic measures in their objective functions, establishing the need to go beyond classical measures and precisely focus on redundant information (common knowledge; defined in Section~\ref{sec:preliminaries}) between spurious and invariant subgraphs. We incorporate the redundant information $\rdn{Y}{\hat{G}_s,\hat{G}_c}$ about target label $Y$ between the learned invariant subgraph $\hat{G}_c$ and spurious subgraph $\hat{G}_s$ into the learning objective for robust and generalized graph classification (see Proposed Optimization~\ref{propopt2}). Finally, we introduce an alternating optimization to solve our learning objective that alternates between: (i) estimating the redundant information term (which itself requires an optimization on its lower bound); and (ii) maximizing it along with additional desired objectives. Our main contributions can be summarized as follows: 
\begin{itemize}[leftmargin=*,topsep=0pt,itemsep=0pt]
    \item We establish a theoretical connection between SCMs and Partial Information Decomposition by analyzing canonical examples, offering a new lens to understand information flow in causal graph learning. 
    \item We propose a novel multi-level optimization framework that we call – Redundancy-guided Invariant Graph learning (RIG) - that leverages redundant information between the invariant and spurious subgraphs to achieve out-of-distribution (OOD) generalization on graphs.
    \item We perform comprehensive experiments on both synthetic and real-world datasets to validate our insights and demonstrate the effectiveness of our proposed framework across 4 synthetic and 7 real-world datasets, including Two-piece graph datasets \citep{gala}, DrugOOD \citep{ji2023drugood}, and CMNIST \citep{IRM}.
\end{itemize}

\textbf{Related Works:} \textit{Invariant Graph Learning} has generated significant interest for improving OOD generalization on graphs. \citet{dir} propose an invariant subgraph learning algorithm (DIR) which conducts interventions on the training distribution to obtain causal rationales while filtering out spurious patterns. \citet{ciga} propose an information-theoretic objective (CIGA) to extract the desired invariant subgraphs which are immune to distribution shifts. \citet{gala} propose  Graph invAriant Learning Assistant (GALA) that incorporates an assistant model that needs to be sensitive to graph environment changes or distribution shifts to learn invariant graphs. \citet{fan2022debiasing} introduces a general disentangled GNN framework (DisC) to learn the causal substructure and bias substructure, respectively. \citet{li2022learning} design a GNN-based subgraph generator (GIL) to extract invariant subgraphs, then uses the complementary variant subgraphs to infer latent environment labels, followed by an invariant learning module to improve generalization to unseen graphs.

\textit{Graph data augmentation} aim to enrich the training distribution by introducing perturbations to the node features and graph structures \citep{ding2022data}. \citet{sui2023unleashing} propose a data augmentation strategy, Adversarial Invariant Augmentation (AIA), to address covariate distribution shifts on graphs. \citet{liu2022graph} introduce a new augmentation operation called environment replacement, which automatically creates virtual data examples to improve rationale identification. Similarly, \citet{kong2022robust} propose FLAG (Free Large-scale Adversarial Augmentation on Graphs), an approach that iteratively augments node features with gradient-based adversarial perturbations during training to enhance OOD performance. \textit{Our novelty lies in leveraging a new information-theoretic tool called PID for a more nuanced understanding of spuriousness in graphs, and also incorporating a PID term, redundant information, into an alternating optimization for improved OOD generalization.}

\textit{Partial Information Decomposition}~\citep{williams2010nonnegativePID, venkatesh2024gaussian,goswami2023computing,pakman2021estimating,lyu2024explicit} is an active area of research, with growing applications in neuroscience and machine learning~\citep{tax2017partial,dutta2020information,hamman2023demystifying,ehrlich2022partial,liang2023multimodal,wollstadt2023rigorous,mohamadi2023more,dutta2021fairness,diffusionPID,dissanayake2024quantifying,haldertowards,dutta2023review}. However, the use of PID terms as regularizers in the graph domain is largely unexplored. Only a few studies, e.g.,~\citet{dissanayake2024quantifying} have attempted to incorporate PID terms as regularizers but not for graphs. To the best of our knowledge, we are the first to incorporate redundant information into invariant graph learning objective.

\section{Preliminaries}
\label{sec:preliminaries}
\begin{figure}[!ht]
    \centering
\includegraphics[width=0.8\linewidth]{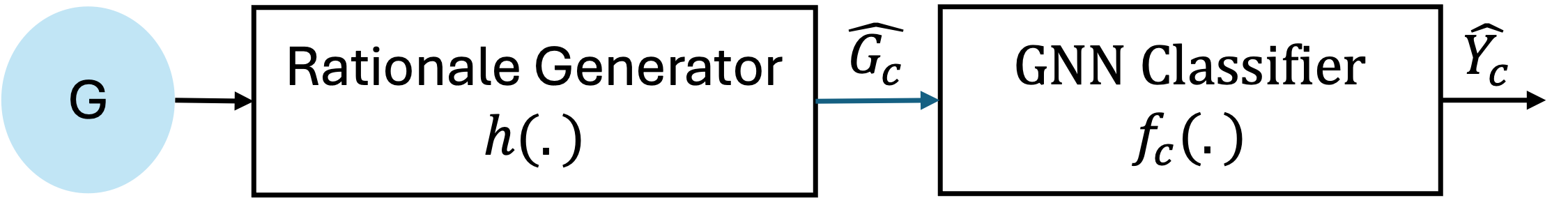}
    \caption{Causally-aligned graph neural network. \label{fig:your_label}}
\end{figure}
In this work, we are interested in out-of-distribution (OOD) generalization in graph classification. Consider a collection of graph datasets $\mathcal{D} = \{ \mathcal{D}^e \}_{e \in \mathcal{E}_{\text{all}}},$ collected from multiple environments $\mathcal{E}_{\text{all}}$, each with slightly shifted distributions, e.g., different locations from where data is collected. The random variable $E$ denotes the environment. Samples $(G_i^e, Y_i^e) \in \mathcal{D}^e$ from the same environment $E=e$ are assumed to be independent and identically distributed (i.i.d.) with a distribution $\mathbb{P}^e$. A causally aligned Graph Neural Network (GNN) model $\rho = f_c \circ h $ typically consists of a rationale generator $h: \mathcal{G} \rightarrow \mathcal{G}_c$ that attempts to learn a meaningful causal subgraph $\hat{G}_c$ for each graph $G$, and a GNN classifier with a classification head $f_c:\mathcal{G}_c \rightarrow \mathcal{Y}$ that predicts the label $\hat{Y}_c$ based on the estimated $\hat{G}_c$, where $\mathcal{G}$ is the graph space and $\mathcal{Y}$ is the target space (see Fig.~\ref{fig:your_label}). To denote parameters of a particular block of the model architecture, we use subscripts, e.g., $h_\beta$ or $f_{\theta_c}$ where the blocks are parameterized by $\beta$ or $\theta_c$, respectively. Our \textbf{goal} is to train a GNN model with graph data from the training environment $\mathcal{D}_{tr} = \{ \mathcal{D}^e \}_{e \in \mathcal{E}_{\text{tr}}\subseteq\mathcal{E}_{\text{all}}}$ that generalizes well to unseen environments during inference. We denote the true and estimated causal subgraphs as \( G_c \) and \( \hat{G}_c \), and the spurious ones as \( G_s \) and \( \hat{G}_s \), respectively. Similarly, the estimated causal and spurious predictions are denoted by $\hat{Y}_c$ and $\hat{Y}_s$, respectively.

\textbf{Background on PID:} The classical measure of the total information that two random variables $A$ and $B$ jointly contain about a target variable $Y$ is given by mutual information $\mut{Y}{A,B}$
(see~\citet{cover2012elements} for a comprehensive background). Mutual information $\mut{Y}{A,B}$ is defined as the Kullback–Leibler (KL) divergence~\citep{cover2012elements} between the joint distribution $P_{YAB}$ and the product of the marginal distributions $P_Y \otimes P_{AB}$, and is equal to zero if and only if $(A, B)$ is statistically independent of $Y$. \emph{Intuitively, this quantity captures the total predictive signal about $Y$ that is jointly present in $(A, B),$ i.e., how well one can learn or infer $Y$ from the pair $(A, B)$.} 

However, classical mutual information $\mut{Y}{A,B}$ does not disentangle the contribution of $A$ and $B$ individually, e.g., what is uniquely contributed by each or redundantly shared between them. To this end, an emerging body of work in information theory called Partial Information Decomposition (PID) \citep{williams2010nonnegativePID} goes beyond classical measures, and disentangles the joint information content $\mut{Y}{A,B}$ about a target variable $Y$ shared among multiple random variables $A$ and $B$ into four non-negative quantities (see Fig. \ref{fig:PID}) as follows:
\begin{align} \label{eq:PID}
\mut{Y}{A,B} &= \uni{Y}{B|A} + \uni{Y}{A|B} \notag \\  & + 
\rdn{Y}{A,B}+ \syn{Y}{A,B}.
\end{align}
Here, redundancy $\rdn{Y}{A,B}$ is the information about $Y$ that is shared by both $A$ and $B$; uniqueness, $\uni{Y}{A|B}$ and $\uni{Y}{B|A}$, denotes the information uniquely provided by $A$ or $B$, respectively; and synergy, $\syn{Y}{A,B}$, captures the information about $Y$ that emerges only when $A$ and $B$ are both present together.
One of the well accepted PID definitions proposed by \citet{bertschinger2014quantifying} is given below: 
\begin{definition}[Unique information \citep{bertschinger2014quantifying}]
\label{def_brojaRedUni} Let $\Delta$ be the set of all joint distributions on $(Y, A, B)$ and $\Delta_P$ be the set of joint distributions with same marginals on $(Y,A)$ and $(Y,B)$ as the true distribution $P_{YAB}$, i.e., $\Delta_P=\{Q_{YAB}{\in} \Delta$: $Q_{YA}=P_{YA}$ and $Q_{YB}=P_{YB}\}$. Then, \begin{align} 
\uni{Y}{A|B}&:=\min _{Q \in \Delta_P} \mutd{Q}{Y}{A|B}.
\end{align}
Here, $\mutd{Q}{Y}{A|B}$ is the conditional mutual information under joint distribution $Q_{YAB}$ instead of $P_{YAB}$.
\end{definition}
\begin{figure}
  \centering
  \includegraphics[width=0.47\columnwidth]{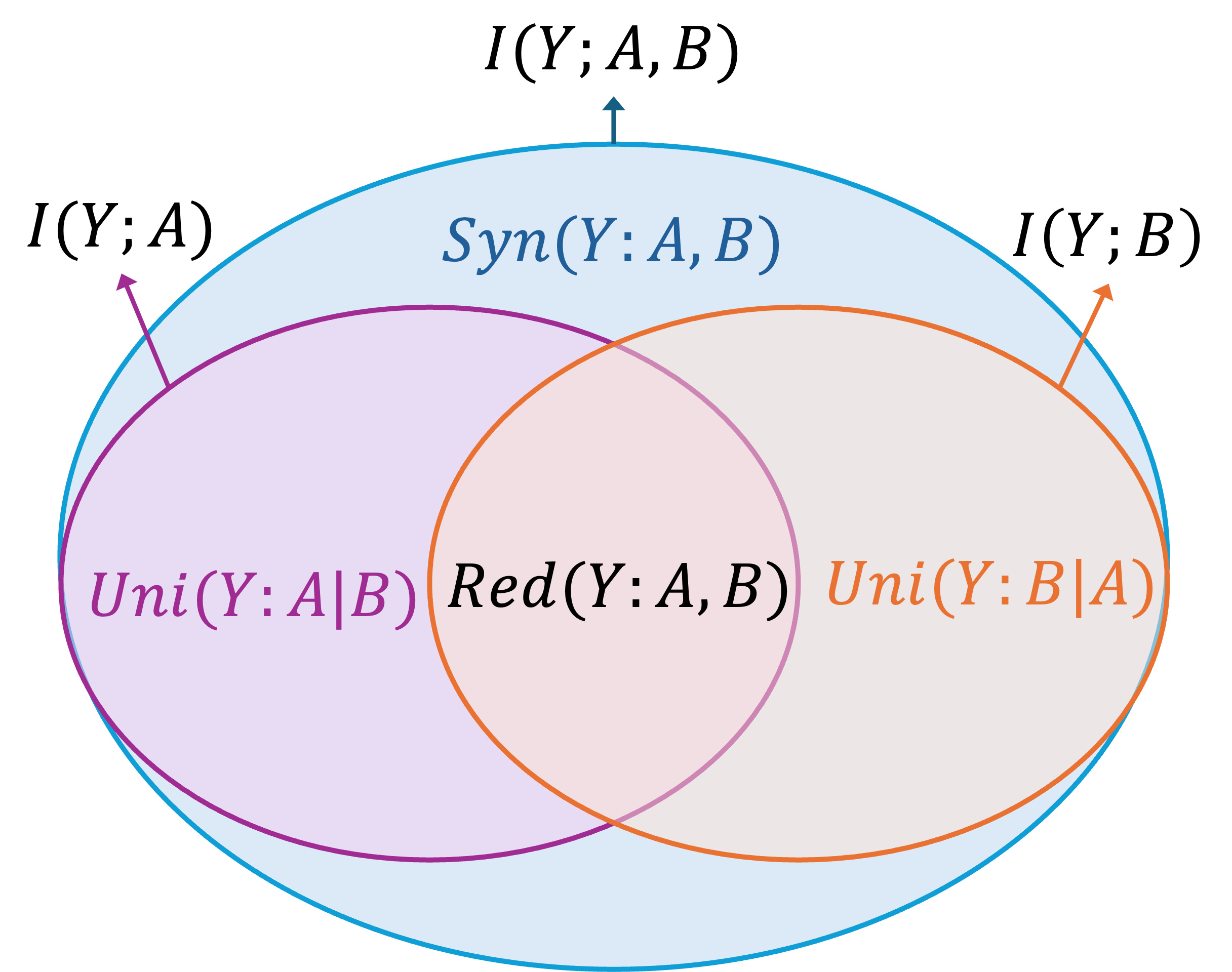}
  \caption{Decomposition of $\mut{Y}{A,B}$.}
  \label{fig:PID}
\end{figure}
Interestingly, defining any one of the PID terms suffices to obtain the others due to the following relationship among the PID terms \citep{bertschinger2014quantifying}: \begin{equation}\mut{Y}{A} = \uni{Y}{A|B}+ \rdn{Y}{A,B}. \label{eq:mi_decomp}\end{equation}
Essentially, $\rdn{Y}{A,B}$ can be interpreted as the sub-volume between $\mut{Y}{A}$ and $\mut{Y}{B}$ (see Fig. \ref{fig:PID}). Hence, $\rdn{Y}{A,B} = \mut{Y}{A} - \uni{Y}{A \mid B}.$ Finally, synergy can be expressed as: \begin{align}
\syn{Y}{A,B} =\; & \mut{Y}{A,B} - \uni{Y}{A|B} \notag \\& - \uni{Y}{B|A} - \rdn{Y}{A,B}, \end{align} which can be computed once both unique and redundant information terms have been obtained.

\begin{figure}[t] 
    \centering
    \begin{subfigure}{0.31\columnwidth}
        \centering
        \includegraphics[width=\linewidth]{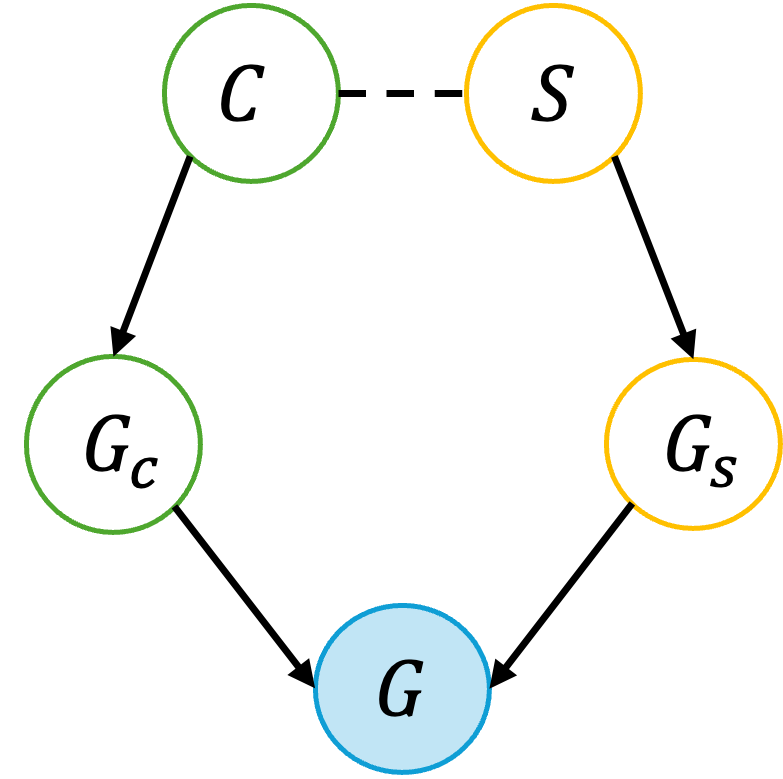}
        \caption{SCM}
        \label{fig:SCM}
    \end{subfigure}
    \hfill
    \begin{subfigure}{0.31\columnwidth}
        \centering
        \includegraphics[width=\linewidth]{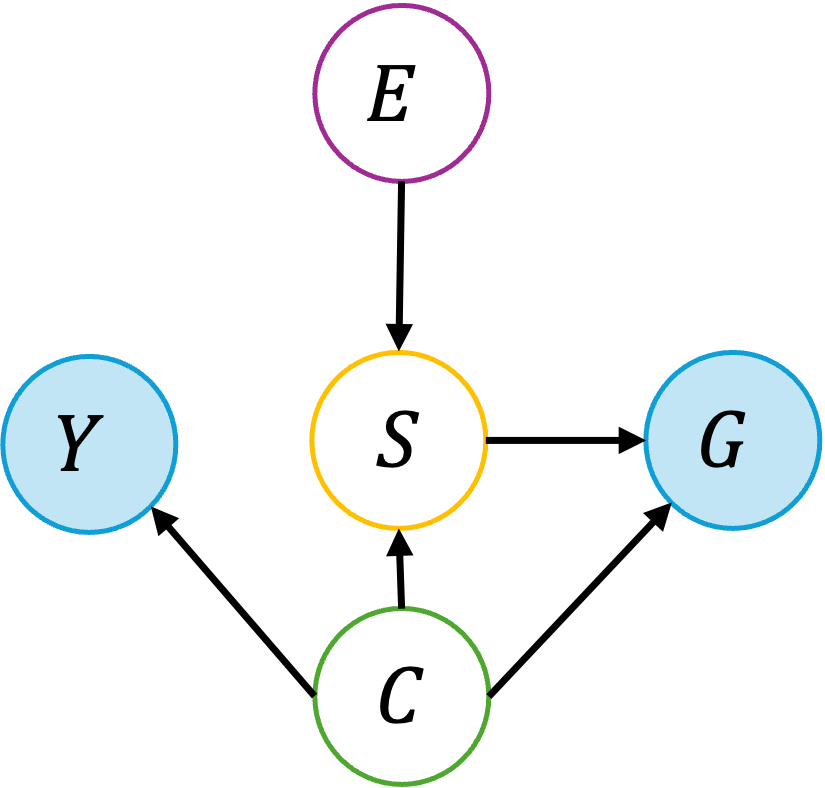}
        \caption{FIIF}
        \label{fig:FIIF}
    \end{subfigure}
    \hfill
    \begin{subfigure}{0.31\columnwidth}
        \centering
        \includegraphics[width=\linewidth]{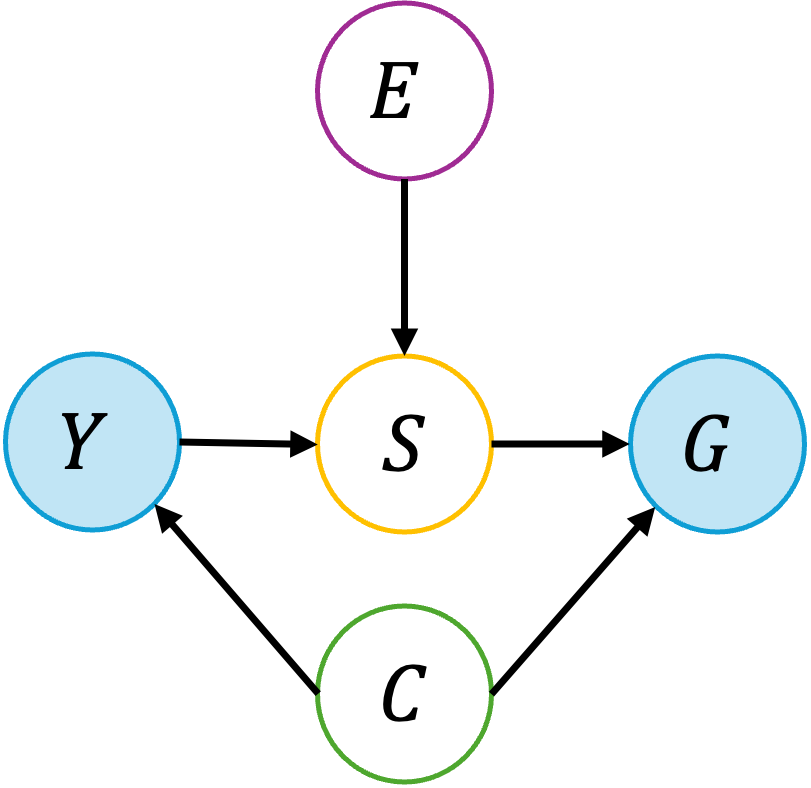}
        \caption{PIIF}
        \label{fig:PIIF}
    \end{subfigure}
    \caption{Graph generation with distribution shifts.}
    \label{fig:all}
\end{figure}

\textbf{Causal Modeling for OOD Graph Generation:} Here, we describe a causal view of the graph generation process. As in previous works \citep{gala,ciga,ahuja2021invariance}, we assume that a graph is generated through a mapping $f_{gen}:\mathcal{Z}\rightarrow\mathcal{G}$ where $\mathcal{Z} \subseteq \mathbb{R}^n$ consists of unobserved, latent variables. We assume that the latent variable from $\mathbb{Z}$ can be decomposed into an invariant part $C \in \mathbb{R}^{n_c}$ (that is not affected by the environment $E$), and a varying (spurious) part $S \in \mathbb{R}^{n_s}$ that is affected by $E$. Here, $n = n_c + n_s$. 

We focus on graph classification tasks. We use a Structural Causal Model (SCM) \citep{pearl2009causality} which captures causal relationships among four key variables: the input graph $G$, the ground-truth label $Y$, the causal part $C$, and the spurious part $S$. Fig.~\ref{fig:SCM} illustrates the SCM where each link denotes a causal relationship between two variables, and the dashed arrow indicates additional dependencies. Here, we assume that $C$ and $S$ control the generation of the observed subgraphs as follows: $G_c := f^{G_c}_{\text{gen}}(C)$, $G_s := f^{G_s}_{\text{gen}}(S)$, and the complete graph $G := f^{G}_{\text{gen}}(G_c, G_s)$. 

Next, we model the interactions between $C$ and $S$ using two types of SCMs: (i) Fully Informative Invariant Features (FIIF) (see Fig. \ref{fig:FIIF}); and (ii) Partially Informative Invariant Features (PIIF) (see Fig. \ref{fig:PIIF}). The two causal models differ depending on the informativeness of the causal part $C$ about the label $Y$ \citep{gala}. For FIIF, $C$ is fully informative of $Y$, i.e., $Y\indep S|C$ and $S$ is directly controlled by $C$. In contrast, for PIIF, we have $Y \not\indep S|C,$ and $S$ is indirectly controlled by $C$ through $Y$. The formal definitions are as follows, where noises are omitted for simplicity:
 \begin{align*}
&\text{(FIIF)} \  Y = f_{\text{inv}}(C); \ 
S = f_{\text{spu}}(C, E); \
G = f_{\text{gen}}(C, S). \\
&\text{(PIIF)}\  Y = f_{\text{inv}}(C); \ 
S = f_{\text{spu}}(Y, E); \ G = f_{\text{gen}}(C,S).\
\end{align*} where $f_{\text{inv}}$ indicates the actual labeling process where label $Y$ for graph $G$ is assigned based on $C$ and $f_{\text{spu}}$ describes how $S$ is affected by $C$ and $E$.

\textbf{Causally-Aligned GNNs:}
Inspired by the concept of structural alignment in CIGA~\citep{ciga,gala}, a causally aligned GNN has two distinct components: (a) a featurizer GNN $h: \mathcal{G} \rightarrow \mathcal{G}_c$ that aims to extract an invariant subgraph \( \hat{G}_c= h(G) \) aligned with the causal substructure \( G_c \); and (b) a classifier GNN \(f_c: \mathcal{G}_c \rightarrow \mathcal{Y} \) that predicts the label \( \hat{Y}_c = f_c(\hat{G}_c) \) based on the learned \( \hat{G}_c \). 
Formally, the objective is to learn \( h \) and \( f_c \) as follows:
\begin{equation}
\max_{f_c, h} \ \mut{\hat{G}_c}{Y} \ \text{such that} \ \hat{G}_c \indep E, \ \hat{G}_c = h(G). \label{eq:objective}
\end{equation}
Here, \( \mut{\hat{G}_c}{Y} \) is the mutual information between the learned invariant subgraph and the target label. The independence constraint $\hat{G}_c {\indep} E$ is to ensure robustness across environments \( E \). However, enforcing independence $\hat{G}_c {\indep} E$  is difficult in practice due to the lack of information about the environment $E$~\citep{IRM,rex,ciga,gala}. Several studies address this challenge by augmenting environment information, incorporating graph information bottleneck, or proposing a contrastive framework \citep{liu2022graph,wu2022handling,dir,ciga,miao2022interpretable1,gala}. 

CIGA \citep{ciga} has two objectives for invariant GNN learning using classical information theory.
\begin{align}
\text{(CIGAv1)}
\max_{f_c, h}\mut{\hat{G}_c}{Y} \ 
\text{s.t.} \ \hat{G}_c \in \argmax{\substack{\hat{G}_c = h(G),\\|\hat{G}_c| \leq p_c}} \mut{\hat{G}_c}{ \hat{G}'_c \mid Y}. \label{eq:cigav1}
\end{align}
Here, $\hat{G}_c = h(G)$, $\hat{G}'_c = h({G}')$, $p_c$ is a size constraint imposed on $\hat{G}_c$, and $G' \sim P(G \mid Y)$, i.e., $G'$ is sampled from the training graphs that share the same label $Y$ as $G$, anticipating that $G$ and ${G}'$ would belong to two different environments. However, simply maximizing $\mut{\hat{G}_c}{Y}$ does not guarantee that the learned features are actually invariant. $\hat{G}_c$ can still contain content from true $G_s$, especially when $G_s$ is spuriously correlated with $Y$ for both FIIF and PIIF, leading to another proposition:
\begin{align}
\text{(CIGAv2)} \ & \max_{f_c, h} \  \mut{\hat{G}_c}{Y} + \mut{\hat{G}_s}{Y}  \nonumber \\ \text{s.t.} & \ 
 \hat{G}_c \in \argmax{\hat{G}_c = h(G)} \mut{\hat{G}_c}{\hat{G}'_c \mid Y}, \nonumber \\
& \mut{\hat{G}_s}{Y} \leq \mut{\hat{G}_c}{Y},
\hat{G}_s = G - h(G). \label{eq:cigav2}
\end{align}

Another related work GALA~\citep{gala} propose an alternative modification to CIGAv1 Eq.~\eqref{eq:cigav1}. They find a new proxy environment assistant model $A$ that samples $\hat{G}_c$ by assuming there exists a subset of training data where $P(Y |G_s)$ varies, while $P(Y |G_c)$ remains invariant, i.e., reduced spuriousness dominance. Incorporating samples from this subset could potentially invalidate the dominance of $G_s$. So, they modify the constraint in the CIGAv1 objective as follows: 
\begin{equation}
\hat{G}_c \in \argmax{\hat{G}^p_c}\ \mut{\hat{G}^p_c}{\hat{G}^n_c \mid Y}.
\end{equation}
Here, \(\hat{G}^p_c=h(G^p)\) where $G^p$ is sampled from a subset dominated by spurious correlations, and \(\hat{G}^n_c=h(G^n)\) is from a subset where invariant correlations prevail over spurious ones. They sample these subsets using an assistant model $A$, typically prone to spurious correlations. 
The subsets are then selected depending on whether $A$'s predictions are correct or not. Let: 
\begin{align*}
\{ \hat{G}_c^p \} = \{ h(G_i^p) \mid A(G_i^p) = Y_i \}, \\
\{ \hat{G}_c^n \} = \{ h(G_i^n) \mid A(G_i^n) \neq Y_i \}.  
\end{align*}

In our work, we identify key limitations of these existing approaches that rely on maximizing classical information-theoretic measures, e.g., $\mut{\hat{G}_c}{Y}$ and $\mut{\hat{G}_s}{Y}$. We will show that the learned subgraph \(\hat{G}_c\) may not faithfully capture the true $G_c$ in both the FIIF and PIIF causal models. 
These limitations motivate us to propose a new objective function that goes beyond classical information-theoretic measures and is based on PID, as we discuss next.

\section{Main Contributions}

\begin{prop}
\label{prop1}
The total predictive information that the invariant variable $C$ and the spurious variable $S$ contain about the target variable $Y$ decomposes into four nonnegative terms:
\begin{align}
\mut{Y}{C,S} &= \uni{Y}{C|S} + \uni{Y}{S|C} \notag \\& + 
\rdn{Y}{C,S}+ \syn{Y}{C,S}.
\end{align}
\end{prop}
We now demonstrate how unique and redundant information relate to the underlying graph generation process and the interaction between latent variables.

\begin{lemma}[FIIF]
Under the FIIF assumption, the true spurious variable $S$ does not have any unique information about the target variable $Y$, i.e., $\uni{Y}{S|C} =0$, but $S$ and $C$ may have redundant information $\rdn{Y}{C,S}$. 

\label{lemma:FIIF}
\end{lemma}
\begin{proof} 
From the definition of FIIF (Fig. \ref{fig:FIIF}), $C$ is fully informative of $Y$, i.e., $Y {\indep} S|C$. Thus, $I(Y;S|C) = 0$.  Then, Definition~\ref{def_brojaRedUni} gives $\uni{Y}{S|C} =0$. 

Now, from Eq.~\eqref{eq:mi_decomp}, $\rdn{Y}{C,S} = \mut{Y}{S} - \uni{Y}{S|C} = \mut{Y}{S}$ which is positive as long as there is a significant dependence between $Y$ and $S$. 
\end{proof}

To illustrate this nuanced scenario under FIIF, we provide an example. Let $S {=} C+N$, $Y {=} C $ where $N$ is Gaussian $ \mathcal{N}(0, \sigma^2_N)$ and $N{\indep} Y$. Here, $\mut{Y}{S|C}=\mut{C}{C+N|C}= H(C|C) {-} H(C|C+N,C)= 0.$ Now, $\uni{Y}{S|C} \leq \mut{Y}{S|C} = 0$. But, $\rdn{Y}{C,S} = \mut{Y}{S} {-} \uni{Y}{S|C} {=} \mut{Y}{S} >0$. 

Lemma~\ref{lemma:FIIF} highlights that under the FIIF assumption, the true spurious graph $G_s$ (from $S$) might only have redundant information about the target variable $Y$, but no unique information. Thus, maximizing $\mut{\hat{G}_s}{Y}$ (as done in CIGAv2 Eq.~\eqref{eq:cigav2}) which is a sum of both $\uni{Y}{\hat{G}_s|\hat{G}_c}$ and $\rdn{Y}{\hat{G}_s,\hat{G}_c}$ can be misleading, causing deviation from converging to the true $G_c$ and $G_s$ in the FIIF setting. \emph{We contend that one should instead leverage PID to precisely focus on the term $\rdn{Y}{\hat{G}_s,\hat{G}_c}$ rather than the whole of $\mut{\hat{G}_s}{Y}$ to avoid maximizing the $\uni{Y}{\hat{G}_s|\hat{G}_c}$ term.}

\begin{lemma}[PIIF]
\label{lemma:PIIF}
Under the PIIF assumption, the true spurious variable $S$ may have more, equal, or less information about $Y$ than $C$, i.e., $\mut{S}{Y}$ may be greater, less, or equal to $\mut{C}{Y}$. Thus, $\uni{Y}{S|C}$ can be greater, less, or equal to $\uni{Y}{C|S}$.
\end{lemma}
\begin{proof}
To prove this result, we provide an example that aligns with the definition of PIIF (Fig. \ref{fig:PIIF}). Let $S = Y+N_s,$ and $Y= C+N_c$ where noise $N_s$ and $N_c$ are standard Gaussian noises with $N_s\sim \mathcal{N}(0, \sigma^2_{N_s})$, $N_c\sim \mathcal{N}(0, \sigma^2_{N_c})$ and ${N_s}\indep Y$, ${N_c}\indep Y$. Now, if $\sigma^2_{N_c} \gg \sigma^2_{N_s}$, then $\mut{Y}{S} > \mut{Y}{C}$ (see Lemma \ref{app:lemma3} in Appendix \ref{app:lemma3}). Similarly, one can also choose the variances $\sigma^2_{N_c}$ and $\sigma^2_{N_s}$ in a manner that leads to the other criterion $\mut{S}{Y} \leq \mut{C}{Y}$. From the definition of PID, $\mut{Y}{S} = \uni{Y}{S|C}  + \rdn{Y}{S,C}$ and $\mut{Y}{C} = \uni{Y}{C|S} + \rdn{Y}{S,C}.$ If $\mut{Y}{S} > \mut{Y}{C}$, we therefore have: $\uni{Y}{S|C}  + \rdn{Y}{S,C} > \uni{Y}{C |S} + \rdn{Y}{S,C}.$ This leads to $\uni{Y}{S|C} > \uni{Y}{C|S}$. 
\end{proof}

From Lemma~\ref{lemma:PIIF}, we further contend that the constraint $\mut{Y}{\hat{G}_s} \leq \mut{Y}{\hat{G}_c}$ (as in CIGAv2 Eq.~\eqref{eq:cigav2}) can be misleading in the PIIF setting, deviating the objective from converging to the true $G_s$ and $G_c$. Intuitively, enforcing the inequality can unintentionally push $\hat{G}_c$ to include parts of true $G_s$. To more precisely control the influence of $G_s$ on $\hat{G}_c$, we propose maximizing only redundant information instead of $\mut{Y}{\hat{G}_s}$ and also eliminate the constraint $\mut{Y}{\hat{G}_s} \leq \mut{Y}{\hat{G}_c}$. Instead, we propose the following optimization problem:
\begin{propopt} 
\label{propopt2}
For a graph distribution and GNN model with a rationale generator $h$ and classifier $f_c$, our optimization objective is:
\begin{align} \label{eq:rgala}
(RIG)\ & \max_{f_c,h} \mut{Y}{\hat{G}_c} + \rdn{Y}{\hat{G}_c,\hat{G}_s} \nonumber \\ & \text{s.t.} \ \hat{G}_c \in \argmax{\hat{G}^p_c} \ \mut{\hat{G}^p_c}{\hat{G}^n_c \mid Y}. 
\end{align}
Here, $\hat{G}_s = G - h(G)$ is the estimated spurious subgraph. Also, \( \hat{G}^p_c \in \{ \hat{G}^p_c = h(G^p) \} \), and \( \hat{G}^n_c \in \{ \hat{G}^n_c = h(G^n) \} \) are the estimated invariant subgraphs. 
\end{propopt}

\subsection{RIG: Our Proposed Framework for Invariant Graph Learning}
\label{sec:framework}

Solving the objective function in Proposed Optimization~\ref{propopt2} is nontrivial since it involves redundant information, and estimating $\rdn{Y}{\hat{G}_c,\hat{G}_s}$ itself requires solving an additional optimization problem. \emph{To make this optimization problem tractable in practice, we introduce an alternating optimization strategy that iteratively alternates between estimating redundant information and maximizing the objective in Eq.~\eqref{eq:rgala}.} This procedure helps disentangle misleading information from invariant subgraphs, thereby enhancing out-of-distribution generalization.

\begin{figure}[t] 
    \centering
    \begin{subfigure}{0.92\columnwidth}
        \centering
        \includegraphics[width=\linewidth]{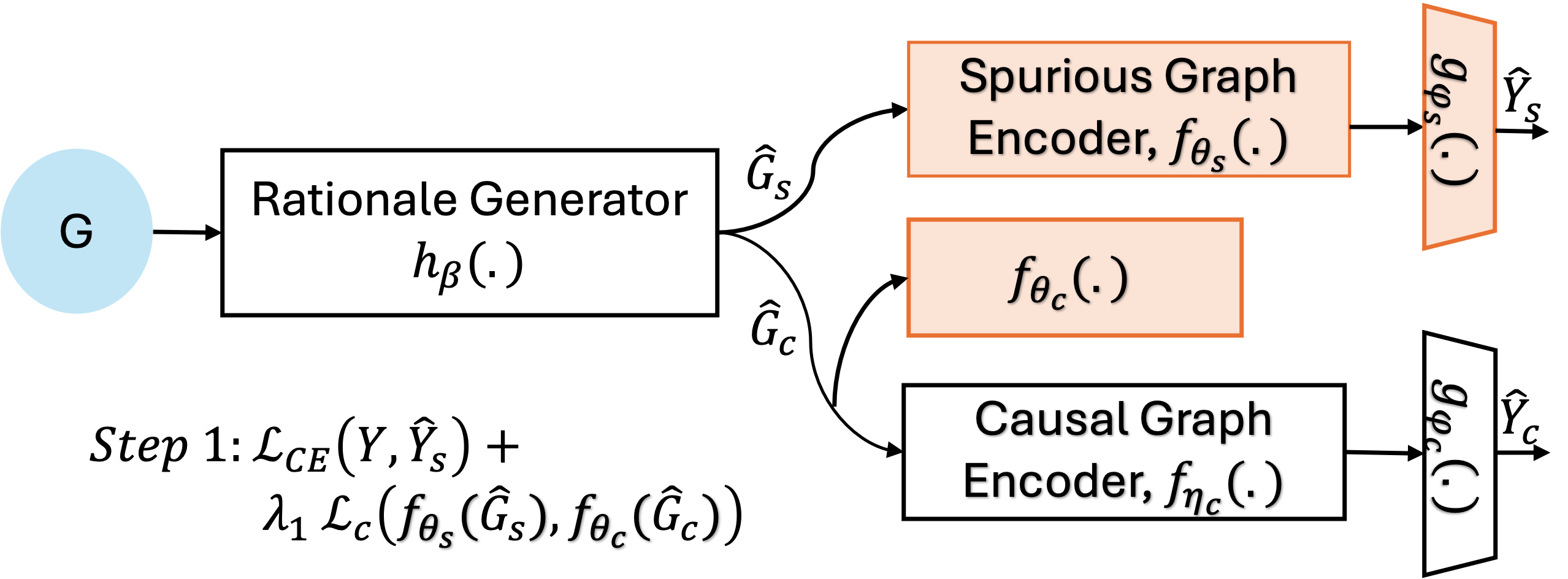}
        \caption{Step 1: The parameters $\theta_s$, $\theta_c$, and $\phi_s$ are updated during training, while other parameters remain fixed.}
        \label{fig:step1}
    \end{subfigure}
    \vspace{2mm}
    \begin{subfigure}{0.92\columnwidth}
        \centering
        \includegraphics[width=\linewidth]{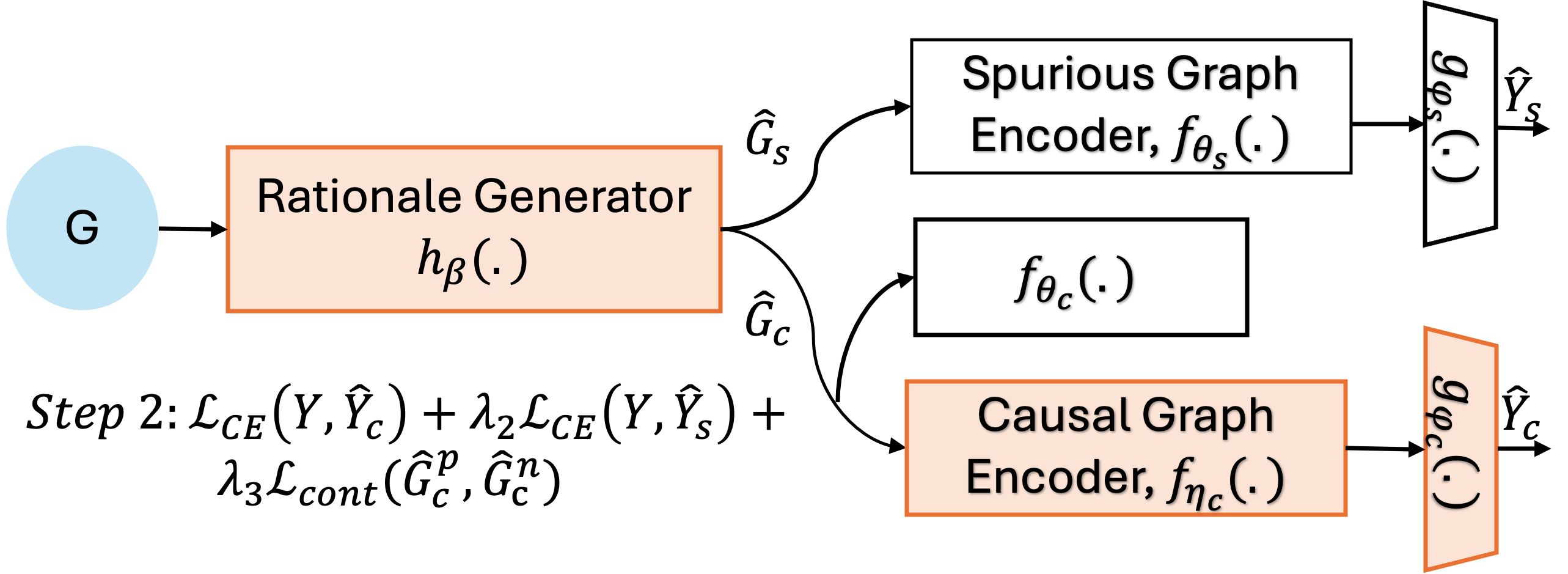}
        \caption{Step 2: The parameter $\beta$, $\eta_c$, and $\phi_c$ are updated during training, while the rest is kept frozen.}
        \label{fig:step2}
    \end{subfigure}
    \caption{Proposed redundancy-based invariant graph learning framework. Highlighted in orange are the components that are being updated in each step.}
    \label{fig:framework}
\end{figure}

\RestyleAlgo{ruled}
\SetKw{KwInit}{Initialize}
\SetKwInOut{KwData}{Input}
\SetKwInOut{KwOut}{Output}

\begin{algorithm}[!ht]
\caption{Redundant Information-based Invariant Graph Learning}
\KwData{Training data $\mathcal{D}_{tr}$; environment assistant $A$; rationale generator $h$; encoder $f$; classification head $g$; warm-up epochs $e_w$; epoch lengths $e_1$ and $e_2$; maximum training epochs $e$; batch size $b$.}

Initialize environment assistant $A$.

\For{$i \in \{1,\dots,e\}$}{
    \If{$i< e_w$}{
        Calculate $\mathcal{L}_{\text{CE}}(Y,\hat{Y}_c)$;\\
        Update the parameters $\beta, \theta_s, \phi_s, \theta_c, \eta_c,$ and $\phi_c$ via gradients to optimize Eq. \ref{eq:step0}
        \tcp*{Step 0}}
    \Else{
     cycle = $((i - e_w) \bmod  (e_1+e_2);$\\
     \If{$cycle < e_1 $}{
        Calculate $\mathcal{L}_{r}$ in Eq.~\eqref{eq:step1};\\
        Update the parameters $\theta_c$, $\theta_s$ and $\phi_s$ via gradients to minimize Eq.~\eqref{eq:step1} \tcp*{Step 1}}
    \Else{
    Sample a batch of data $\{G_i, Y_i\}_{i=1}^b$ from $\mathcal{D}_{tr}$\;
    Obtain predictions $\{\hat{y}_e^{(i)}\}_{i=1}^b$ using k-means clustering on the subgraphs by $A$\;
    
    \For { each sample $(G_i, Y_i) \in \{G_i, Y_i\}_{i=1}^b$}{
    Find \emph{positive graphs} ${G^p}$ with the same $Y_i$ but different $\hat{y}_e^{(i)}$\;
    Find \emph{negative graphs} ${G^n}$ with different $Y_i$ but the same $\hat{y}_e^{(i)}$\;
    Calculate the objective in Eq.~\eqref{eq:step2}\;
    Update the parameters $\beta,\eta_c,$ and $\phi_c$ via gradients optimizing Eq. \ref{eq:step2}\;
    \tcp{Step 2}
    }
    }}
}
\KwOut{final model $g_{\phi_c}\circ f_{\eta_c}\circ h_\beta$}
\label{algo:rig}
\end{algorithm}

\textbf{Optimization Objective.} 
We begin by designing a GNN architecture with two branches (see Fig.~\ref{fig:framework}) to separately capture the causal component \( G_c \) and the spurious component \( G_s \) from the input graph \( G \). The model consists of two main components: (i) a shared rationale generator \( h_\beta: \mathcal{G} \rightarrow {\mathcal{G}}_c, {\mathcal{G}}_s \), which decomposes the input graph into an estimated causal subgraph \( \hat{G}_c \) and a non-causal (spurious) subgraph \( \hat{G}_s \); and (ii) a GNN classifier $f_c$ which we decompose into a GNN encoder $f$ and a classification head $g$. The encoder \( f \) maps the subgraph to a representation, and the classifier head \( g \) makes predictions based on this representation. The final prediction $\hat{Y}_c$ is produced by the causal branch through the classifier \( g_{\phi_c} \), which operates on the estimated causal subgraph \( \hat{G}_c \). A parallel branch uses the spurious subgraph \( \hat{G}_s \) to make dummy predictions \( \hat{Y}_s \) from it.

Next, we propose a three-step optimization framework to solve Eq.~\ref{eq:rgala} (see Algorithm \ref{algo:rig} and Fig. \ref{fig:framework}).

\textbf{Step 0: Warm Up.} In this step, we train the proposed architecture in Fig.~\ref{fig:framework} using the following unconstrained objective: $\max_{f_c,h} \; \mut{Y}{\hat{G}_c}$. In practice, this can be implemented by minimizing a standard classification loss, such as the cross-entropy loss \citep{GIB}, defined as:
\begin{equation}
\label{eq:step0}
\min_{\{\beta, \theta_s, \phi_s, \theta_c, \eta_c, \phi_c\}} \mathcal{L}_{\text{CE}}(Y,\hat{Y}_c).
\end{equation}
where $\hat{Y}_c$ is the predicted label based on the estimated invariant representation $\hat{G}_c$. Due to the model's inherent tendency toward shortcut learning, the initial representation $\hat{G}_c$ may include components from the spurious subgraph $G_s$.

 \textbf{Step 1: Estimating Redundant Information.} Next, we freeze all parameters except $\theta_c$, $\theta_s$ and $\phi_s$, and estimate redundant information about $Y$ that is embedded in $\hat{G}_c$ and $\hat{G}_s$. For this estimation, we first observe that the redundant information $\rdn{Y}{\hat{G}_s,\hat{G}_c}$~\citep{bertschinger2014quantifying} is lower-bounded by a term called intersection information~\citep{griffith2014intersection,griffithRedInfo} denoted by $\rdncap{Y}{\hat{G}_s,\hat{G}_c}$ that is easier to estimate. Since our objective is to maximize $\rdn{Y} {\hat{G}_s,\hat{G}_c}$, maximization of the lower bound $\rdncap{Y}{\hat{G}_s,\hat{G}_c}$ serves our purpose \citep{dissanayake2024quantifying}. Therefore, we first estimate the intersection information from \citet{griffith2014intersection}, which is defined as follows:  

\begin{definition}[$I_\cap$ measure \citep{griffith2014intersection}]
\label{def:griffithRed}
\begin{equation}
\begin{split}
&\rdncap{Y}{\hat{G}_s,\hat{G}_c} = \max_{P(Q|Y)} \mut{Y}{Q} \ \\ &\text{s.t.}  \exists f_{\theta_c}, f_{\theta_s} \text{ with } Q = f_{\theta_c}(\hat{G}_c)=f_{\theta_s}(\hat{G}_s).
\end{split}
\end{equation}
\end{definition} Here $f_{\theta_c}$, $f_{\theta_s}$ are deterministic functions and $Q$ is a random variable capturing the shared information component between $\hat{G}_s$ and $\hat{G}_c$. For practical implementation, we select $Q$ in Definition~\ref{def:griffithRed} as $Q = f_{\theta_s}(\hat{G}_s)$, where both $f_{\theta_s}(\cdot)$ and $f_{\theta_c}(\cdot)$ are parameterized using GNNs.  With the substitution $Q = f_{\theta_s}(\hat{G}_s)$, Definition~\ref{def:griffithRed} leads to the following optimization problem:
\begin{equation}
\label{eq:red}
\max_{\theta_s, \theta_c} \mut{Y}{f_{\theta_s}(\hat{G}_s)} \ \text{s.t.} \ {f_{\theta_s}(\hat{G}_s) = f_{\theta_c}(\hat{G}_c)}.
\end{equation}

We approximately maximize $\mut{Y}{f_{\theta_s}(\hat{G}_s)}$ by minimizing the cross-entropy loss between $\hat{Y}_s$ and $Y$, where $\hat{Y}_s = g_{\phi_s} (f_{\theta_s}(\hat{G}_s))$ and the constraint is added as a regularizer. In effect, we minimize the following loss function with respect to $\theta_s$, $\theta_c$, and $\phi_s$:\begin{equation}
\label{eq:step1}
\mathcal{L}_r(\theta_s, \theta_c, \phi_s) = 
 \mathcal{L}_{CE}(Y, \hat{Y}_s) + \lambda_1 \mathcal{L}_{c}(f_{\theta_s}(\hat{G}_s), f_{\theta_c}(\hat{G}_c)) 
\end{equation}


Here, $\lambda_1$ is a positive hyperparameter and $\mathcal{L}_{c}(f_{\theta_s}(\hat{G}_s),f_{\theta_c}(\hat{G}_c)) = \frac{1}{U V}\sum^U_{u=1}\sum^V_{v=1}D^2_{u,v}$ where $D = f_{\theta_s}(\hat{G}_s)-f_{\theta_c}(\hat{G}_c) \in \mathbb{R}^{U\times V}$. The $\mathcal{L}_{c}$ term enforces $f_{\theta_s}(\hat{G}_s) \approx f_{\theta_c}(\hat{G}_c)$, so that the constraint in Eq.~\eqref{eq:red} can be satisfied. Solving optimization Eq.~\eqref{eq:step1} ultimately leads to a rough estimate of the intersection information at the end of this step: $\rdncap{Y}{\hat{G}_s,\hat{G}_c} \approx \mut{Y}{Q} = \mut{Y}{f_{\theta_s}(\hat{G}_s)}$.

It may be noted that the module $f_{\theta_c}(\cdot)$ is incorporated as a separate channel to estimate $Q$, without interfering with the functionality of the causal graph encoder $f_{\eta_c}(\cdot)$ (see Fig.~\ref{fig:framework}).

\textbf{Step 2: Maximizing Objective.} Finally, we freeze $\theta_s$, $\theta_c$, and $\phi_s$ and minimize the following loss function to effectively maximize the desired objective in Proposed Optimization~\ref{propopt2}.
\begin{align}
\label{eq:step2}
\mathcal{L}(\beta,\eta_c, \phi_c) 
&= \mathcal{L}_{CE}(Y, \hat{Y}_c) 
+ \lambda_2 \mathcal{L}_{CE}(Y, \hat{Y}_s) \notag \\
&\quad + \lambda_3 \mathcal{L}_{cont}(\hat{G^p_c}, \hat{G^n_c}).
\end{align}
Here, $\lambda_2$ and $\lambda_3$ are positive scalar hyperparameters and $\mathcal{L}_{CE}(Y, \hat{Y}_c)$ is minimized as a proxy for maximizing $\mut{Y}{\hat{G}_c}$ as per Proposed Optimization~\ref{propopt2} (Eq.~\eqref{eq:rgala}). Similarly, minimizing the cross-entropy loss $\mathcal{L}_{CE}(Y, \hat{Y}_s)$ now effectively promotes the maximization of 
redundant information $\rdn{Y}{\hat{G}_c,\hat{G}_s}$ as per Proposed Optimization~\ref{propopt2} since: (i) $\rdn{Y}{\hat{G}_c,\hat{G}_s}
\geq \rdncap{Y}{\hat{G}_s,\hat{G}_c}$;  (ii) Step 1 ensured that $\rdncap{Y}{\hat{G}_s,\hat{G}_c} \approx \mut{Y}{f_{\theta_s}(\hat{G}_s)}$; and (iii) Since $\hat{Y}_s = g_{\phi_s} (f_{\theta_s}(\hat{G}_s))$.  


\begin{table*}[!htbp]
\centering
\small
\setlength{\tabcolsep}{4pt}
\caption{Test performance (\%) on real-world graphs with complex distribution shifts (mean $\pm$ std).}
\label{tab:drugood}
\begin{tabular}{lccccccc}
\toprule
\textbf{Methods} & \textbf{EC50-Assay} & \textbf{EC50-Scaffold} & \textbf{EC50-Size} & \textbf{Ki-Assay} & \textbf{Ki-Scaffold} & \textbf{Ki-Size} & \textbf{CMNIST} \\
\midrule
ERM  & 69.34\textsmaller{$\mathbin{\pm}$ 2.35} & 62.12 \textsmaller{$\mathbin{\pm}$2.73} & 62.39\textsmaller{$\mathbin{\pm}$1.03} & 73.72\textsmaller{$\mathbin{\pm}$2.22} & 68.31\textsmaller{$\mathbin{\pm}$2.42} & 73.84\textsmaller{$\mathbin{\pm}$4.35} & 20.82\textsmaller{$\mathbin{\pm}$3.82} \\
GREA  & 71.15{\scriptsize$\pm$2.09} & 63.79{\scriptsize$\pm$1.00} & 60.32{\scriptsize$\pm$1.53} & 72.52{\scriptsize$\pm$3.80} & 63.86{\scriptsize$\pm$7.35} & 69.24{\scriptsize$\pm$3.78} & 14.75{\scriptsize$\pm$2.45} \\
GSAT  & 75.23{\scriptsize$\pm$2.33} & 65.56{\scriptsize$\pm$0.35} & 62.85{\scriptsize$\pm$1.07} & 72.78{\scriptsize$\pm$1.86} & 72.59{\scriptsize$\pm$1.52} & 72.50{\scriptsize$\pm$1.25} & 16.16{\scriptsize$\pm$3.42} \\
GIL   & 70.85{\scriptsize$\pm$2.59} & 62.93{\scriptsize$\pm$1.02} & 63.19{\scriptsize$\pm$1.91} & 77.00{\scriptsize$\pm$1.16} & 72.81{\scriptsize$\pm$0.95} & 74.78{\scriptsize$\pm$1.82} & 14.43{\scriptsize$\pm$3.06} \\
CAL   & 76.45{\scriptsize$\pm$2.82} & 66.10{\scriptsize$\pm$1.01} & 63.28{\scriptsize$\pm$1.55} & 74.06{\scriptsize$\pm$5.65} & 71.30{\scriptsize$\pm$1.50} & 74.21{\scriptsize$\pm$2.64} & 33.45{\scriptsize$\pm$13.56} \\
CIGAv2& 74.31{\scriptsize$\pm$1.40} & 65.80{\scriptsize$\pm$1.06} & 64.05{\scriptsize$\pm$0.33} & 77.58{\scriptsize$\pm$2.32} & 71.53{\scriptsize$\pm$1.01} & 70.19{\scriptsize$\pm$7.25} & 22.56{\scriptsize$\pm$12.29} \\
GALA  & 76.43{\scriptsize$\pm$2.06} & 65.54{\scriptsize$\pm$1.55} & 63.93{\scriptsize$\pm$1.13} & 77.81{\scriptsize$\pm$2.58} & 73.81{\scriptsize$\pm$1.64} & 76.80{\scriptsize$\pm$2.51} & 68.95{\scriptsize$\pm$0.45} \\
RIG (ours)   & \textbf{76.78{\scriptsize$\pm$1.77}} & \textbf{67.20{\scriptsize$\pm$0.92}} & \textbf{64.20{\scriptsize$\pm$1.48}} & \textbf{78.42{\scriptsize$\pm$1.26}} & \textbf{74.16{\scriptsize$\pm$1.18}}& \underline{76.53{\scriptsize$\pm$1.35}} & \textbf{69.00{\scriptsize$\pm$0.66}} \\
\bottomrule
\end{tabular}
\end{table*}

Lastly, $\mathcal{L}_{cont}$ is the contrastive loss, as defined in~\citet{gala}. $\mathcal{L}_{cont}$ approximates the conditional mutual information $\mut{\hat{G}^p_c}{\hat{G}^n_c \mid Y}$ in the constraint of Proposed Optimization~\ref{propopt2} (see Appendix \ref{app:objective} for more details). To obtain proper subsets $\{G^p\}$ and $\{G^n\}$, following~\citet{gala}, we implement an assistant model $A$ in Algorithm \ref{algo:rig} using ERM (Empirical Risk Minimization). Since ERM tends to learn the most dominant features, the assistant model $A$ tends to often rely on spurious subgraphs $G_s$ to make predictions $\hat{Y}$. Based on this behavior, we define $\{G^p\}$ as the set of samples for which $A$ predicts the correct label, and $\{G^n\}$ as the set of samples where $A$'s prediction is incorrect. We continue to alternate between steps 1 and 2 (see Algorithm \ref{algo:rig}) and effectively optimize Eq.~\eqref{eq:rgala} in Proposed Optimization~\ref{propopt2}.

\section {Empirical Results}

We conduct extensive experiments on four synthetic and seven real-world datasets, including the Two-piece graph datasets \citep{gala}, DrugOOD \citep{ji2023drugood}, and CMNIST \citep{IRM}, to evaluate the effectiveness of our proposed optimization framework, RIG. We compare RIG with several baselines, including GREA \citep{liu2022graph}, GSAT \citep{miao2022interpretable2}, CAL \citep{causalattention}, GIL \citep{li2022learning}, CIGAv2 \citep{ciga}, and GALA \citep{gala}. Details of the datasets and experimental setup are provided in Appendix~\ref{app:experiments}. 

\textbf{OOD Performance Analysis:} We report classification accuracy for the Two-piece graph and CMNIST datasets, and ROC-AUC for the DrugOOD datasets, as in related work~\citep{gala}. We repeat each evaluation five times using different random seeds and select models based on their validation performance. We report the mean and standard deviation (std) of the corresponding metric in Table~\ref{tab:drugood} and Table~\ref{tab:spmotif}. 

Table~\ref{tab:drugood} shows the OOD test performance on real-world datasets. Our framework, RIG, outperforms the state-of-the-art (SOTA) baselines on 6 datasets (in bold), including the most challenging DrugOOD-Scaffold benchmarks. 
In the remaining datasets, RIG achieves comparable performance, and for the underlined datasets, it attains low standard deviation, resulting in superior performance when considering (mean $-$ 1$\times$std). Table~\ref{tab:spmotif} presents the out-of-distribution (OOD) test performance on the synthetic Two-piece graph datasets. We observe that as the spurious correlation strength ($b$) increases, e.g., in $\{0.8, 0.9\}$ and $\{0.7, 0.9\}$, our framework RIG consistently outperforms the state-of-the-art (SOTA) baselines. We also observe comparable performance on the $\{0.8, 0.6\}$ and $\{0.8, 0.7\}$ datasets.

\begin{table}[!htbp]
\centering
\caption{Test performance (\%) for Two-piece graph datasets (mean $\pm$ std). Here \{a, b\} refers to the invariant correlation strength and spurious correlation strength, respectively.}
\small
\setlength{\tabcolsep}{1pt}
\label{tab:spmotif}
\begin{tabular}{lcccc}
\toprule
\{$a,b$\} & \{0.8,0.6\} & \{0.8,0.7\} & \{0.8,0.9\} & \{0.7,0.9\} \\
\midrule
ERM        & 77.36{\scriptsize$\pm$0.80}  & 74.64{\scriptsize$\pm$1.70}  & 50.77{\scriptsize$\pm$3.40}  & 42.09{\scriptsize$\pm$2.23} \\
GREA       & 82.81{\scriptsize$\pm$0.68}  & 82.26{\scriptsize$\pm$0.64}  & 49.02{\scriptsize$\pm$3.42}  & 39.52{\scriptsize$\pm$2.14} \\
GSAT       & 81.25{\scriptsize$\pm$0.38}  & 79.12{\scriptsize$\pm$1.27}  & 46.71{\scriptsize$\pm$2.00}  & 36.45{\scriptsize$\pm$1.04} \\
GIL        & 83.59{\scriptsize$\pm$0.30}  & 82.97{\scriptsize$\pm$0.23}  & 51.62{\scriptsize$\pm$1.02}  & 39.85{\scriptsize$\pm$2.32} \\
CAL        & 73.07{\scriptsize$\pm$6.71}  & 70.58{\scriptsize$\pm$13.65} & 54.03{\scriptsize$\pm$10.07} & 46.87{\scriptsize$\pm$2.94} \\
CIGAv2     & 74.41{\scriptsize$\pm$7.27}  & 70.67{\scriptsize$\pm$12.41} & 49.24{\scriptsize$\pm$7.70}  & 38.57{\scriptsize$\pm$5.20} \\
GALA       & 83.25{\scriptsize$\pm$0.88}  & 81.43{\scriptsize$\pm$0.59}  & 76.51{\scriptsize$\pm$1.93}  & 64.44{\scriptsize$\pm$4.83} \\
RIG (ours) & 83.03{\scriptsize$\pm$0.58}  & 82.05{\scriptsize$\pm$1.36} & \textbf{77.82 {\scriptsize$\pm$1.78}}  & \textbf{65.56{\scriptsize$\pm$4.49}} \\
\bottomrule
\end{tabular}
\end{table}

\textbf{PID Estimation:} To further check if our alternating optimization is indeed maximizing redundant information, we estimate the Partial Information Decomposition (PID) values via convex optimization (Definition~\ref{def_brojaRedUni}). Specifically, we decompose the total information $\mut{Y}{\hat{Y}_s, \hat{Y}_c}$ that the spurious predictions $\hat{Y}_s$ and causal predictions $\hat{Y}_c$ provide about the target $Y$ into four non-negative components: redundancy, unique information in $\hat{Y}_c$ (Uniq\_C), unique information in $\hat{Y}_s$ (Uniq\_S), and synergy. As shown in Fig.~\ref{fig:pid}, RIG exhibits a more balanced decomposition compared to the other two methods, with moderate redundancy and dominant Uniq\_C over Uniq\_S. This indicates that RIG effectively separates both the spurious and invariant graphs, but the spurious graph can have correlation with $Y$. In particular, the dominance of Uniq\_C suggests that the model prioritizes invariant information that is more predictive of the target, thereby capturing the underlying causal structure more accurately (see the accuracies from core and spurious graphs in Table \ref{tab:PID} in Appendix \ref{app:PID} with more details). Appendix includes: \textbf{hyperparameter selection} (Appendix~\ref{app:hyperparameter}), \textbf{ablation study} for different steps (Appendix~\ref{app:ablation}), resource consumptions (Appendix~\ref{app:resource}) and interpretability visualizations (Appendix \ref{app:visualization}).
\begin{figure}[!ht]
    \centering
\includegraphics[width=0.85\linewidth]{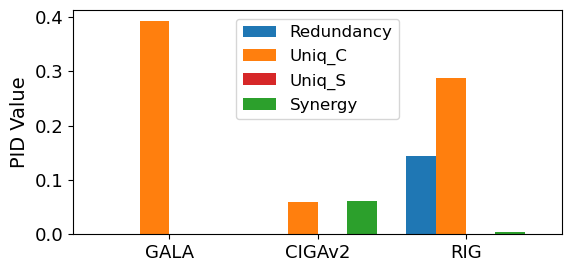}
    \caption{Comparison of PID values across baseline methods with Two-piece dataset \{0.8,0.9\}.}
    \label{fig:pid}
\end{figure}

\textbf{Conclusion:} This work addresses the challenge of learning invariant graph representations for OOD generalization. Leveraging the information-theoretic tool Partial Information Decomposition (PID), we propose \textit{RIG}, a multi-level optimization framework that isolates invariant from spurious components by maximizing redundant information. Experiments on synthetic and real-world datasets demonstrate its effectiveness in improving OOD generalization. Future work will study extensions beyond the graph domain.

\bibliographystyle{abbrvnat}
\bibliography{reference.bib}

\appendix
\newpage
\onecolumn




\section{Limitations}
(i) RIG does not consistently achieve superior performance across all datasets, particularly on the less challenging ones, which may be attributed to the approximation involved in estimating redundant information. Future work will study further improvements in both the estimation and implementation procedures. (ii) The sampling strategy for the contrastive loss relies on the spurious dependencies captured by the assistant model A, which can introduce variability in performance. (iii) A rigorous proof of the optimization convergence remains an important direction for future research. 

\section{Appendix to Theoretical Results}
\label{app:theory}

\begin{lemma}[Noisy Feature]
\label{app:lemma3}
Let $A=Y+N$ where $Y\sim {Bern}(1/2)$ is a random variable taking values $+1$ or $-1$ and the noise $N\sim \mathcal{N}(0, \sigma^2_N)$ is a Gaussian random variable independent of $Y$. Then, mutual information $$\mut{Y}{A} \leq \frac{1}{2} \log_2{\left(1+ \frac{1}{\sigma^2_N} \right)} .$$
\end{lemma}
\begin{proof}
\begin{align}
\mut{Y}{A} = H(A)-H(A|Y)
&=H(Y+N)-H(Y+N|Y)\\
&=H(Y+N)-H(N|Y)\\
&=H(Y+N)-H(N), \text{  since } N\indep Y\\
&\overset{(a)}{\leq} \frac{1}{2} \log_2 2\pi e\left(1+\sigma_N^2\right)-\frac{1}{2} \log_2 2\pi e\left(\sigma_N^2\right)\\
&=\frac{1}{2} \log_2 \left( 1 + \frac{1}{\sigma_N^2} \right).
\end{align}
Here, (a) holds because the entropy of $Y+N$ is bounded by $\frac{1}{2} \log_2 2\pi e\left(1+\sigma_N^2\right)$ (proved in \citet[Theorem 8.6.5]{cover2012elements}). We also refer to \citet[Chapter 9]{cover2012elements} for a discussion on Gaussian channels.
\end{proof}

If we keep the distribution of $Y$ fixed and vary the noise variance $\sigma_N^2$, then we will observe a decreasing trend of $\mut{Y}{A}$ with increasing $\sigma_N^2$.  Fig.\ref{Fig:MIvsN} shows the exact trend where $Y$ is a Bernoulli random variable.

\begin{figure*}[htbp!]
   \centering
    \includegraphics[height=130pt]{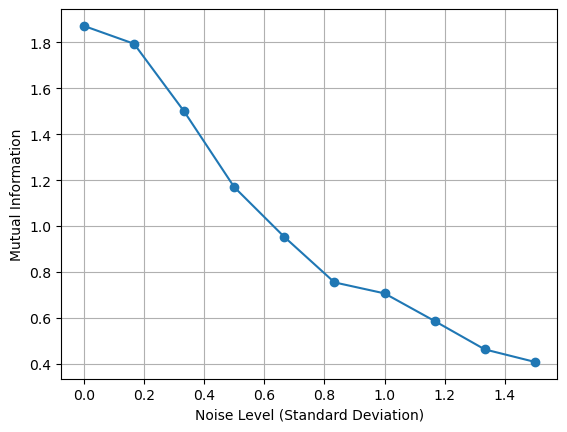} 
    \caption{Mutual Information vs. Noise Level ($Y$ is Bernoulli)}
     \label{Fig:MIvsN}
\end{figure*}

\section{Appendix to Practical Implementation of the Objective}

\label{app:objective}

As the estimation of mutual information is highly expensive, we adopt contrastive learning $\mathcal{L}_{cont}(\hat{G^p_c}, \hat{G^n_c})$ to approximate $\mut{\hat{G}^p_c}{\hat{G}^n_c \mid Y}$ \citep{ciga,gala}. We define $\mut{\hat{G}^p_c}{\hat{G}^n_c \mid Y}$ as follows:

$$ \mut{\hat{G}^p_c}{\hat{G}^n_c \mid Y} \approx \mathbb{E}_{\substack{\{ \hat{G}^p_c, \hat{G}^n_c \} \sim \mathcal{P}_h(G \mid \mathcal{Y}=Y),\\ \{ \hat{G}_c^i \}_{i=1}^M \sim \mathcal{P}_h(G \mid \mathcal{Y} \neq Y)}} 
\left[
\log \frac{
    e^{\Phi(f_{\hat{G}^p_c}, f_{\hat{G}^n_c})}
}{e^{\Phi(f_{\hat{G}^p_c}, f_{\hat{G}^n_c})} + \sum_{i=1}^{M} {e^{\Phi(f_{\hat{G}_c}, f_{\hat{G}^i_c})}}
}\right] $$

Here, \((\hat{G}^p_c, \hat{G}^n_c)\) are subgraphs extracted by \(h\) from sets \(\{G^p\}\) and \(\{G^n\}\), respectively, which share the same label \(Y\). The set \(\{\hat{G}^c_i\}_{i=1}^M\) consists of subgraphs extracted by \(h\) from graphs \(G^i\) that have labels different from \(Y\). \(\mathcal{P}_h(G \mid \mathcal{Y} = Y)\) denotes the push-forward distribution of \(P(G \mid \mathcal{Y} = Y)\) through the rationale generator \(h\), where \(P(G \mid \mathcal{Y} = Y)\) is the conditional distribution of graphs given a label \(Y\), and \(P(G \mid \mathcal{Y} \neq Y)\) is the conditional distribution given a label different from \(Y\). The subgraphs \(\hat{G}_c = h(G)\), \(\hat{G}^p_c = h(\hat{G}^p)\), \(\hat{G}^n_c = h(\hat{G}^n)\), and \(\hat{G}^i_c = h(G^i)\) are the outputs of the rationale generator \(h\), and their representations \(f_{\hat{G}_c}\), \(f_{\hat{G}^p_c}\), \(f_{\hat{G}^n_c}\), and \(f_{\hat{G}^i_c}\) are the embeddings of these subgraphs. The function $\Phi$ denotes a similarity measure between representations. As \(M \to \infty\), $\mathcal{L}_{cont}(\hat{G^p_c}, \hat{G^n_c})$ approximates $\mut{\hat{G}^p_c}{\hat{G}^n_c \mid Y}$ \citep{wang2020understanding}.

\section{Appendix to Experiments}
\label{app:experiments}
\textbf{Datasets:} 

\textit{Two-piece graph datasets.} Two-piece graph datasets \citep{gala} are three-class synthetic datasets based on BAMotif \citep{luo2020parameterized}. The task is to identify which of the three motifs — House, Cycle, or Crane — is embedded in each graph. Each dataset is parameterized by two variables \{a,b\}, which control the strength of invariant and spurious correlations, respectively, leading to different relationships between $H(C|Y)$ and $H(S|Y)$. 

\textit{DrugOOD datasets:}  We use six datasets from the DrugOOD benchmark \citep{ji2023drugood}, namely EC50-Assay, EC50-Scaffold, EC50-Size, Ki-Assay, Ki-Scaffold, and Ki-Size, all of which contain core-level annotation noise. The task is to predict ligand-based affinity, with complex distribution shifts arising from variations in assays, molecular scaffolds, and molecule sizes. 

\textit{CMNIST dataset:} We use graph-structured data derived from the ColoredMNIST (CMNIST) dataset~\citep{IRM}. The graphs are generated using the conversion algorithm proposed by~\citet{knyazev2019understanding}, which introduces distinct distribution shifts in the node attributes. The classification task is to determine whether an image contains a digit from $0-4$ or $5-9$. 

In our experiments, we follow standard practices for optimizing GNNs and tuning hyperparameters. The details are provided below.

\textbf{GNN Backbone:}

In line with previous studies \citep{ciga,gala}, we employ the interpretable GNN as the underlying backbone. Formally, given a graph $G$ containing $n$ nodes, a soft mask is predicted as follows: $$Z = \mathrm{GNN}(G) \in \mathbb{R}^{n \times h}, \quad M = a(ZZ^T) \in \mathbb{R}^{n \times n}.$$
where $a$ computes the sampling weights for each edge using a multilayer perceptron (MLP): $M_{ij} = \mathrm{MLP}([Z_i, Z_j])$. Based on the continuous sampling score $M$, $h$ can sample discrete edges according to the predicted scores \citep{miao2022interpretable2}. For each dataset, we sample $r\%$ of all edges, where $r$ is determined based on validation performance; for CMNIST, we follow previous work and fix the ratio to $80\%$.

  
For a fair comparison, we use the same GNN architecture as graph encoders for all methods. By default, we use 3-layer Graph Isomorphism Network (GIN) \citep{xu2018powerful} with Batch Normalization \citep{ioffe2015batch} between layers and Jumping Knowledge (JK) residual connections at the last layer \citep{xu2018representation}. The hidden dimension is set to 32 for the Two-piece and CMNIST datasets, and 128 for the DrugOOD datasets. By default, we use mean pooling over all nodes except DrugOOD datasets, where we follow a 4-layer GIN with sum pooling. 

\textbf{Model Optimization and Selection Criteria:} 

By default, we use the Adam optimizer with a learning rate of
$1e-3$ and a batch size of 32 for all models and all datasets. Except for DrugOOD datasets, we use a batch size of 128, and for CMNIST, we use a batch size of 256 following GALA. To avoid underfitting (Step 0: Warm up), we pretrain models for 10 epochs for all datasets, except for CMNIST, where we pretrain for 5 epochs. To avoid overfitting, we also employ an early stopping of 5 epochs according to the validation performance during Step 2: Maximizing Objective. We use a dropout rate of 0.5 for the CMNIST and DrugOOD datasets. All experiments are repeated with 5 different random seeds. The mean and standard deviation are calculated from 5 runs. 

\textbf{Implementation Details for Baselines:} 

We implement GREA \citep{liu2022graph}, GSAT \citep{miao2022interpretable2}, CAL \citep{causalattention}, GIL~\citep{li2022learning}, CIGA \citep{ciga}, and GALA \citep{gala}, following the  implementation provided in~\citet{gala}. We include some specific details here:

\textit{GREA~\citep{liu2022graph}}. We use a penalty weight of 1 for GREA and the same interpretable ratio as others.

\textit{GSAT~\citep{miao2022interpretable2}}. Following prior work, we use an interpretability ratio of 0.7, a penalty weight of 1, a decay rate of 10\%, and a decay interval set to half of the pretraining epochs. 

\textit{CAL~\citep{causalattention}}. We adopt the same interpretability ratio as previous studies, with the penalty weight selected from \{0.1, 0.5, 1.0\} and choose the one with the best validation performance.

\textit{CIGA~\citep{li2022learning}}. All penalty weights are set according to the authors' recommendation. We do not implement CIGAv1, as GALA represents an improved version of CIGAv1.

For GREA, GSAT, CAL, and CIGAv2 the number of environments is not needed. For GIL, CIGA, GALA, and RIG, the number of environments (used as the number of clusters for GALA) is fixed for the Two-piece and CMNIST datasets, as these values are known: the Two-piece graph dataset contains 3 spurious graph types, and CMNIST contains 2 environments. For DrugOOD datasets, we search the number of environments in the set \(\{2, 3, 5, 10, 20\}\), following previous practice~\citep{yang2022learning}. 

\textit{GIL~\citep{li2022learning}}. We select the penalty weight from $\{1e-5, 1e-3, 1e-1\}$ and interpretability ratio, same as others. 

\textit{GALA~\citep{gala}.} We follow the original GALA framework and use the proposed GNN model to implement their method. For the environment assistant model \(A\), we adopt a vanilla GNN for Two-piece graph datasets, EC50-Size, Ki-Assay, and Ki-Scaffold, and for EC50-Assay, EC50-Scaffold, Ki-Size, and CMNIST, we use XGNN (interpretable GNN backbone). We train \(A\) using only cross-entropy loss. The sampling proxy is constructed based on cluster predictions. We search over penalty weights ($\lambda_3$) \{0.5, 1, 2, 4, 8, 16, 32, 64, 128, 256\} for each dataset and report the best-performing value. 

\textbf{Implementation Details for Our Proposed Method RIG:} 

For a fair comparison, we use the same assistant model, penalty weights \(\lambda_3\), and number of environments as in GALA. The hyperparameter $\lambda_2$ is tuned over \{0.1, 1, 1.5, 2, 2.5, 3, 4\} for all the datasets. We report the best performance obtained in each case. The epoch lengths $e_1$ and $e_2$ are selected $10$ for Two-piece graph datasets, EC50-Assay, and $20$ for Ki-Assay, Ki-Scaffold, and Ki-Size datasets. For CMNIST, EC50-Scaffold, and EC50-Size datasets, $e_1$ is chosen 10 and $e_2$ is chosen 50. The early stopping is deployed after $100$ epochs for the Two-piece graph, Ki-Assay, Ki-Scaffold, and Ki-Size datasets, and $200$ epochs for CMNIST, EC50-Assay, EC50-Scaffold, and EC50-Size datasets. These parameters are selected to ensure stable performance. 

\subsection{Hyperparameter Selection}\label{app:hyperparameter}

To stabilize the optimization of Eq. \eqref{eq:step1}, we replace hyperparameter $\lambda_1$ with a positive learnable parameter $\mu$ and introduce an additional regularization term $(\log \mu)^2$. This enables the algorithm to learn an effective value of $\mu$ and also penalizes both very small values ($\mu \to 0$) and very large values ($\mu \to \infty$), thereby ensuring numerical stability during training.

To evaluate the sensitivity of our proposed method to the hyperparameter \(\lambda_2\) in the objective function (Eq.~\eqref{eq:step2}), we conduct experiments on two of the most challenging datasets: \(\{0.8, 0.9\}\) and EC50-Scaffold. While varying \(\lambda_2\), we keep all other hyperparameters fixed. The results in Fig.~\ref{fig:sensitivity} demonstrate that our method remains mostly stable and robust across different datasets and distribution shifts. 

The hyperparameter $\lambda_3$ is set to the same value as in GALA~\citep{gala}.
\clearpage

\begin{figure}[!htbp]
    \centering
\includegraphics[width=0.5\linewidth]{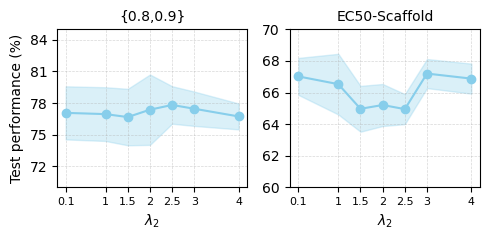}
    \caption{Sensitivity of the model to hyperparameter $\lambda_2$ across different datasets.}
    \label{fig:sensitivity}
\end{figure}

\subsection{Ablation Study} 
\label{app:ablation}
We analyze the effect of each component in our method by comparing three variants. The first variant includes only Step 0 and Step 2, while the second uses Step 1 and Step 2. The third variant incorporates all steps. Table \ref{tab:ablation} shows that combining all steps achieves the best performance, highlighting the complementary contributions of each component.
\begin{table}[!htbp]
\centering
\caption{Ablation study on the Two-piece graph dataset \{0.8, 0.9\} with hyperparameters $\lambda_2 = 2.5$ and $\lambda_3 = 128$.}
\begin{tabular}{c c c c}
\toprule
Step 0 & Step 1 & Step 2 & Performance (\%) \\
\midrule
\checkmark &        & \checkmark & 77.00 {\textsmaller{$\pm$1.95}} \\
          & \checkmark & \checkmark & 76.65 {\textsmaller{$\pm$0.96}} \\
\checkmark & \checkmark & \checkmark & \textbf{77.82 {\textsmaller{$\pm$1.78}}} \\
\bottomrule
\end{tabular}
\label{tab:ablation}
\end{table}

\subsection{PID Estimation} 
\label{app:PID}
We estimate the Partial Information Decomposition (PID) values by considering the ground-truth label $Y$ as the target variable, and the predictions $\hat{Y}_s$ and $\hat{Y}_c$ as the two sources, obtained respectively from the estimated spurious graph $\hat{G}_s$ and the invariant graph $\hat{G}_c$. In this decomposition, \emph{redundancy} represents the information about $Y$ that is shared between $\hat{Y}_s$ and $\hat{Y}_c$; \emph{Uniq\_C} denotes the information about $Y$ that is uniquely captured by $\hat{Y}_c$ but not by $\hat{Y}_s$; \emph{Uniq\_S} refers to the information uniquely captured by $\hat{Y}_s$ but not by $\hat{Y}_c$; and \emph{synergy} corresponds to the information about $Y$ that emerges only when both $\hat{Y}_s$ and $\hat{Y}_c$ are considered jointly. Causal accuracy is evaluated based on the agreement between the ground-truth label $Y$ and the invariant prediction $\hat{Y}_c$, and spurious accuracy is evaluated using $Y$ and the spurious prediction $\hat{Y}_s$. 

\begin{table}[ht!]
\centering
\caption{PID values and test accuracies for two-piece graph dataset$\{0.8,0.9\}$.}
\label{tab:PID}
\begin{tabular}{lcccccc}
\toprule
\{0.8,0.9\} & Redundancy        & Uniq\_C    & Uniq\_S    & Synergy        & Causal Acc. (\%) & Spurious Acc. (\%) \\ \midrule
GALA              & 0          & 0.3924 & 0          & 0          & 76.30       & 33.33    \\
CIGAv2             & 0.0005& 0.0589 & 0   & 0.0620 & 45.70       & 32.60    \\
RIG              & 0.1431 & 0.2877 & 0 & 0.0035 & 78.23       & 56.77   \\ \bottomrule
\end{tabular}
\end{table}

In Table~\ref{tab:PID}, we observe that for GALA the Uniq\_C component is dominant, while the other terms remain zero. This is expected as its optimization objective excludes the terms that come from spurious graphs. In contrast, CIGAv2 exhibits some redundant information, reflecting its training objective to estimate both spurious and invariant graphs. However, the accuracy results suggest that it struggles to balance these components effectively, likely due to high bias. On the other hand, RIG demonstrates a more balanced decomposition, with both redundancy and Uniq\_C outweighing Uniq\_S. This indicates that RIG successfully separates both spurious and invariant graphs. Crucially, since the unique information in $\hat{Y}_c$ dominates Uniq\_S, the model appears to prioritize invariant information more that is predictive of the target, thus capturing the causal structure effectively.

\subsection{Runtime and Resource Usage}
\label{app:resource}
We implement our methods using PyTorch and PyTorch Geometric. All experiments are conducted on an NVIDIA RTX A4500 (CUDA 12.2) and RTX 6000 GPU(CUDA 12.8). We measure the average total training time of both GALA and our proposed method across multiple datasets. For the Two-piece graph dataset \(\{0.7, 0.9\}\), EC50-Scaffold, and CMNIST, our method takes \(768.39 \pm 146.05\), \(2354.1021 \pm 890.4654\), and \(6998.0142 \pm 2066.8555\) seconds, respectively, and GALA requires \(766.11 \pm 64.21\), \(1906.34 \pm 600.79\), and \(5884.53 \pm 7650.35\) seconds, respectively. The runtime varies with the number of scripts running on a single GPU and the type of GPU used.

\subsection{Interpretability Visualization} 
\label{app:visualization}
To enhance the interpretability of the model's outputs, we visualize the edge masks produced by the interpretable GNN backbone using both GALA and our proposed optimization. We use the code provided by \citet{gala}. In these visualizations (Figs. \ref{fig:gala} and \ref{fig:rig}), pink circles denote the nodes of the ground-truth causal subgraph $G_c$, while yellow circles denote the nodes of the ground-truth spurious subgraph $G_s$. The edge color intensity reflects the attention weights assigned by the model, with darker edges indicating higher attention. If the model assigns high attention to the edges connecting the pink nodes, then we can conclude that the invariant subgraph has been correctly identified.

\begin{figure}[htbp!]
\centering
\begin{subfigure}[t]{0.9\linewidth}
    \centering
    \includegraphics[width=\linewidth]{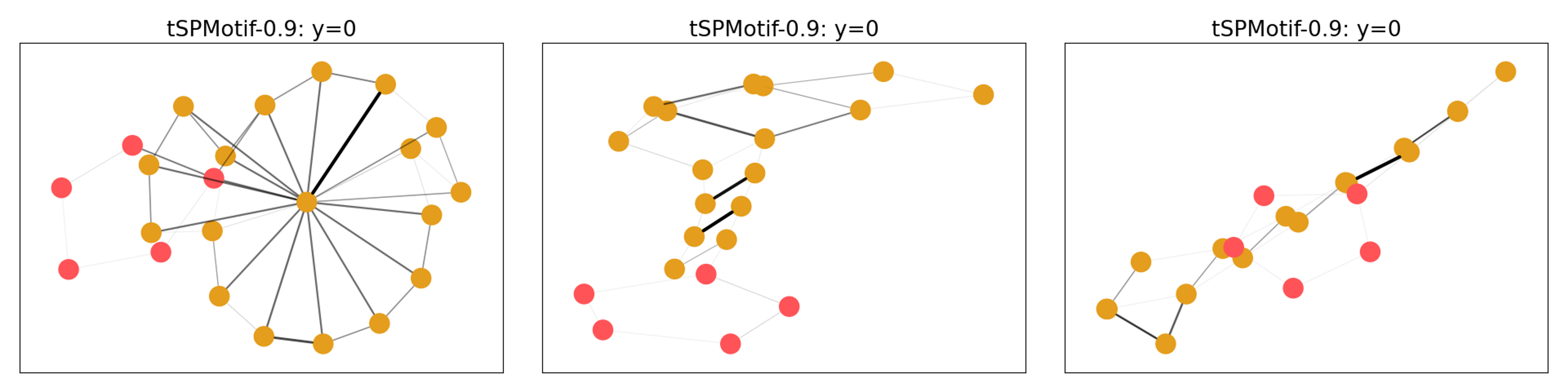}
    \label{fig:gala0}
\end{subfigure}
\begin{subfigure}[t]{0.9\linewidth}
    \centering
    \includegraphics[width=\linewidth]{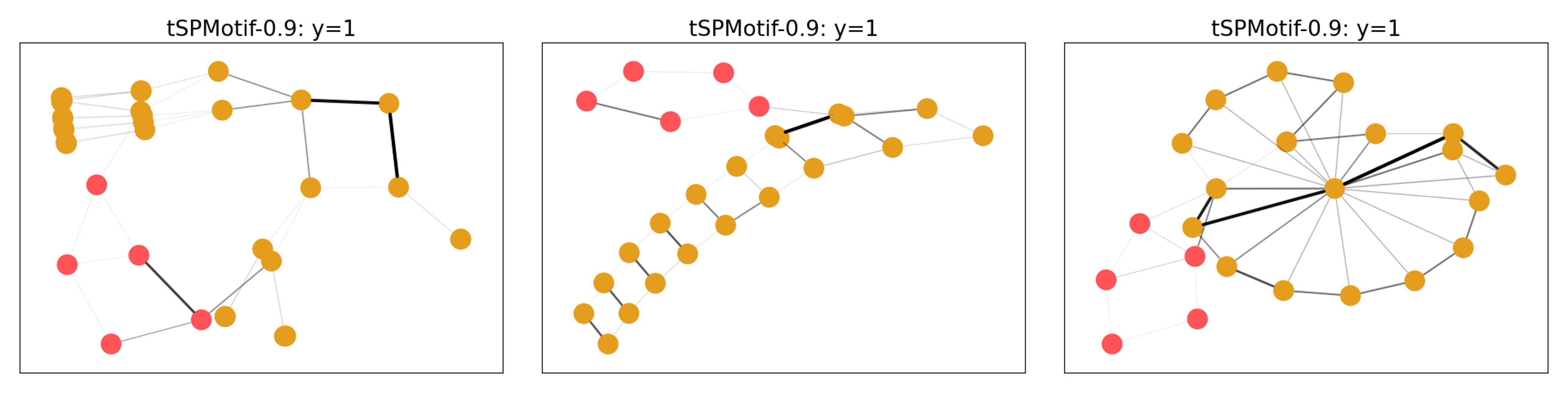}
    \label{fig:gala1}
\end{subfigure}
\begin{subfigure}[t]{0.9\linewidth}
    \centering
    \includegraphics[width=\linewidth]{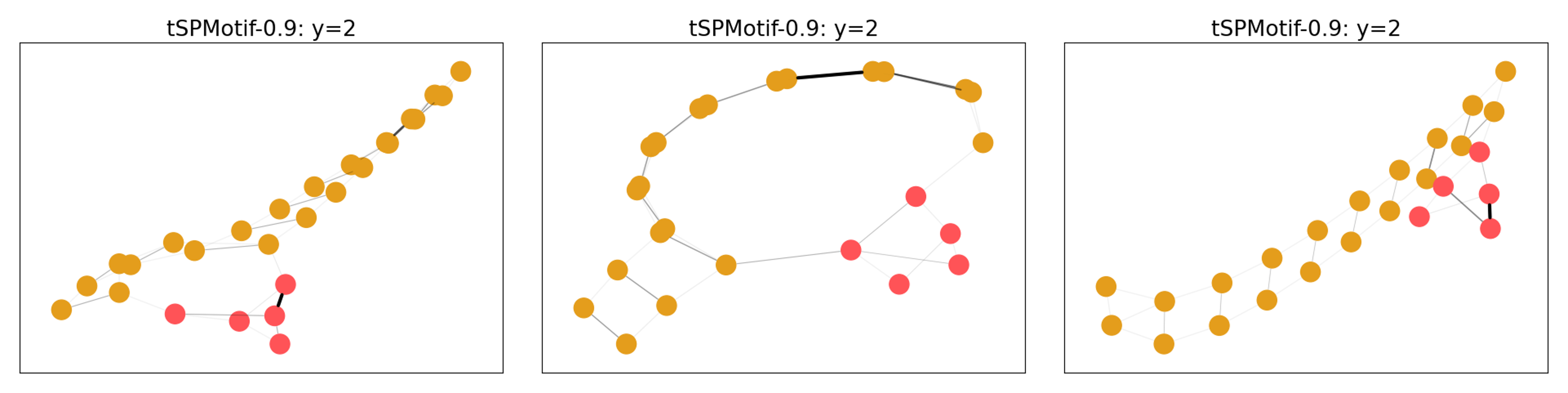}
    \label{fig:gala2}
\end{subfigure}
\caption{Interpretation visualization from the Two-piece graph dataset \{0.8,0.9\} where GALA misclassifies.}
\label{fig:gala}
\end{figure}

Fig.~\ref{fig:gala} presents examples of GALA misclassifications on the {0.8, 0.9} dataset. The model fails to correctly identify the edges connecting the ground-truth nodes. For $y=2$, it identifies one ground-truth edge but still produces an incorrect classification.
The failure might be due to not identifying a sufficient number of edges associated with the important nodes. In Fig.~\ref{fig:rig}, we observe that RIG identifies more edges corresponding to the ground-truth nodes, which might result in a correct prediction for these graphs in contrast to GALA. 

\begin{figure}[htbp!]
\centering
\begin{subfigure}[t]{0.9\linewidth}
    \centering
    \includegraphics[width=\linewidth]{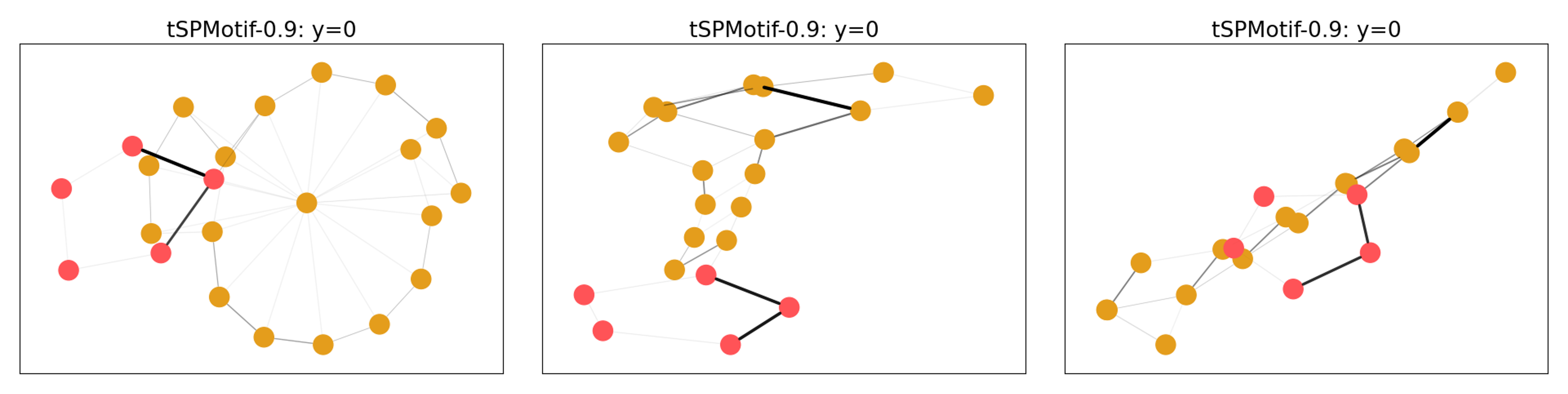}
    \caption{}
    \label{fig:rig0}
\end{subfigure}
\vspace{0.01em}
\begin{subfigure}[t]{0.9\linewidth}
    \centering
    \includegraphics[width=\linewidth]{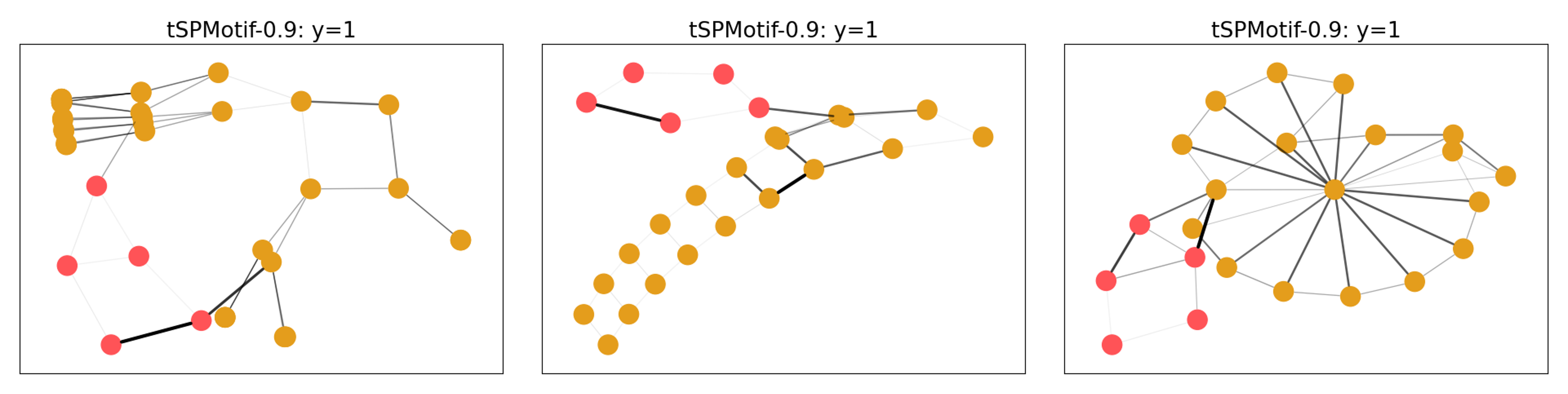}
    \caption{}
    \label{fig:rig1}
\end{subfigure}
\vspace{0.01em}
\begin{subfigure}[t]{0.9\linewidth}
    \centering
    \includegraphics[width=\linewidth]{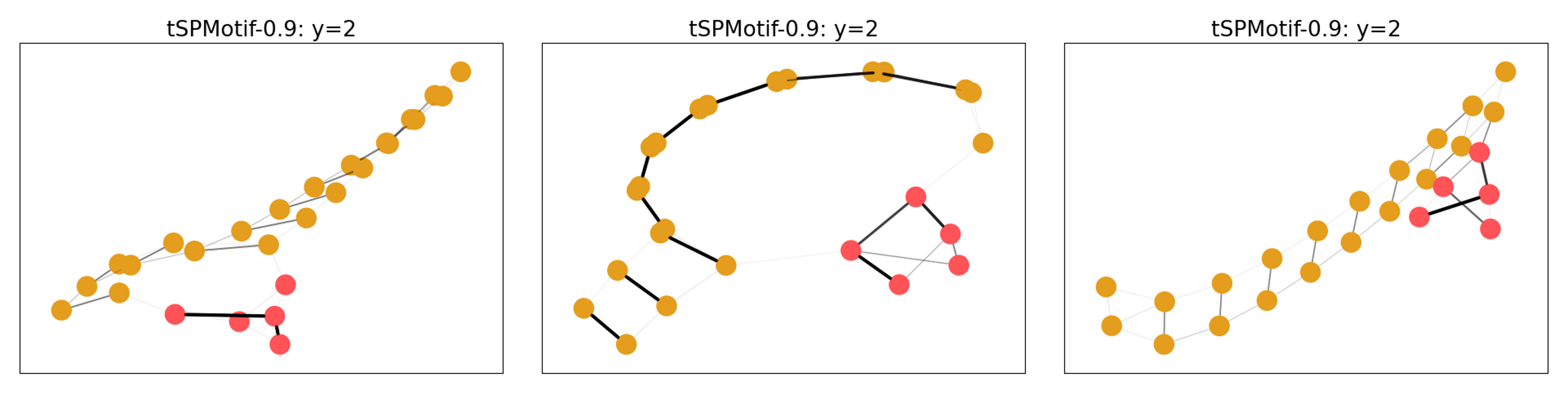}
    \caption{}
    \label{fig:rig2}
\end{subfigure}
\caption{Interpretation visualization from the Two-piece graph dataset \{0.8,0.9\} where RIG classifies the graphs correctly.}
\label{fig:rig}
\end{figure}



\end{document}